\pdfoutput=1

\documentclass[11pt]{article}

\usepackage{acl}

\usepackage{times}
\usepackage{latexsym}
\usepackage{booktabs}
\usepackage{multirow}
\usepackage{amsmath}
\usepackage{amssymb}
\usepackage{bm}
\usepackage{makecell}
\usepackage{graphicx}
\usepackage{subfigure}
\usepackage{pifont}
\usepackage{subfigure}

\usepackage{algorithm}
\usepackage{algpseudocode}

\usepackage[T1]{fontenc}

\usepackage[utf8]{inputenc}

\newcommand{\cmark}{\ding{51}}

\usepackage{microtype}

\usepackage{inconsolata}

\usepackage{amsthm}
\theoremstyle{plain}

\theoremstyle{claim}

\newtheorem*{claim*}{Claim}
\theoremstyle{rethinking}

\theoremstyle{researchquestion}
\newtheorem{researchquestion}{RQ}
\theoremstyle{findings}

\theoremstyle{fact}

\theoremstyle{definition}
\newtheorem*{definition*}{Definition}
\theoremstyle{lemma}

\usepackage{wrapfig,lipsum,booktabs}
\theoremstyle{proper}

\theoremstyle{remark}

\title{  {\color{blue}OVM}, {\color{blue}O}utcome-supervised {\color{blue}V}alue {\color{blue}M}odels for Planning\\ in Mathematical Reasoning}

\author{Fei Yu$^{1,2}$, Anningzhe Gao$^{2}$, Benyou Wang$^{1,2}$ \\
  The Chinese Unviersity of Hong Kong, Shenzhen, China \\
  Shenzhen Research Institute of Big Data \\
  222043013@link.cuhk.edu.cn, gaoanningzhe@sribd.cn,  wangbenyou@cuhk.edu.cn \\
  \url{https://github.com/FreedomIntelligence/OVM}
  }
\begin{document}
\maketitle

\begin{abstract}

Large language models (LLMs) often struggle with maintaining accuracy throughout multiple multiple reasoning steps, especially in mathematical reasoning where an error in earlier steps can propagate to subsequent ones and it ultimately leading to an incorrect answer.
To reduce error propagation, guided decoding is employed to direct the LM decoding on a step-by-step basis. We argue that in guided decoding, assessing the potential of an incomplete reasoning path can be more advantageous than simply ensuring per-step correctness, as the former approach leads towards a correct final answer. This transforms the task into a \textit{value estimation} problem in planning.



Inspired by the findings that \textit{outcome supervision for guided decoding essentially acts as a value model}, we propose Outcome-supervised Value Model (OVM) that employs outcome supervision for training a value model, which prioritizes steps that lead to accurate conclusions. Furthermore, the OVM eliminates the need for labor-intensive annotations of step-level correctness, thereby significantly enhancing its scalability. Our experiments on two multi-step mathematical reasoning datasets, GSM8K and Game of 24, demonstrate the superior performance of the OVM model. Notably, in GSM8K, our \textbf{OVM-7B model achieves state-of-the-art results among LLMs up to  13B parameters}; especially it does not utilize GPT-4 or code execution. These findings offer a novel perspective on the role of outcome supervision in training value models for multi-step reasoning tasks and provide theoretical justification for its advantage in value estimation for guided decoding.
\end{abstract}

\section{Introduction}


Multi-step reasoning problems are challenging for even large language models (LLMs) (\citealp{selection-inference23}; \citealp{compositionality-gap22}; \citealp{cot22}). 
Chain of Thought (CoT) outputs a series of intermediate reasoning steps before the final answer, which significantly improves the performance (\citealp{cot22}; \citealp{cot-bbh23}). 

\textbf{Verifying complete solutions}
Recent studies (\citealp{gsm8k21}, \citealp{deepmind-process22}, \citealp{openai-process23}) have focused on training a verifier, also referred to as a `\textit{reward} model', to verify the correctness of a complete solution among various candidates (\citealp{gsm8k21}).
This training generally involves two types of supervision for training: \textit{outcome supervision} and \textit{process supervision}. 
Recent research has demonstrated a clear advantage of \textit{process supervision} over \textit{outcome supervision} for training reward models in terms of verifying complete reasoning paths (\citealp{openai-process23}). 


\textbf{Guided decoding during intermediate steps}
However, errors often happen during the decoding of intermediate steps, leading to subsequent inaccuracies due to \textit{error propagation}.
For instance, GPT-4 often struggles with the initial step in the Game of 24, yet it can solve the task with multiple attempts (\citealp{tot23}). 
To this end,  guiding language decoding with step-level evaluation has been proposed (\citealp{xie2023guided}; \citealp{grace23}; \citealp{tot23}).
Paralleling the concepts of \textit{rewards} and \textit{values} in reinforcement learning,  the criteria for step-level evaluation could be  either future-agnostic (\citealp{xie2023guided}) or future-oriented  (\citealp{tot23}); the latter (i.e., value models) seems better as it has a longer-horizon perspective.  

\textbf{Value-based guided decoding. }
In line with value-based guided decoding that considers the potential of the possible future-generated solutions, the challenge lies in value estimation. Previous research primarily achieved this through extensive lookahead sampling or simulation to estimate the long-term returns (\citealp{rap2023}; \citealp{CoRe23}; \citealp{tot23}); this introduces an additional decoding cost during the inference of an LLM. An alternative method is to train a value model that enables value estimation during inference without the need for simulation. Inspired by the findings that \textit{outcome supervision for guided decoding essentially acts as a value model}, as found in this paper, we propose the use of outcome supervision to train a value model for value estimation without simulation during inference, called \textbf{OVM}.

Experiments are conducted on two popular multi-step mathematical reasoning datasets -- GSM8K~\citep{gsm8k21} and Game of 24~\citep{tot23}. In GSM8K, our OVM-7B model obtains state-of-the-art performance among models with up to 13B parameters, attaining a 84.7\% accuracy without resorting to supplementary datasets, GPT-4, or executing programs. 
In Game of 24,  OVM-7B reaches 78.7\% success rate with merely 20 nodes visited per step, in stark contrast to its 11\% greedy success rate and 11.7\% with majority voting over 100 reasoning paths. 
Furthermore, we demonstrate that our method attains competitive, and often superior, performance using fewer sampled reasoning paths compared to complete path verification, on both GSM8K and Game of 24. This indicates the effectiveness of OVM in value estimation as a future-oriented evaluation.


In summary, our contributions are three-fold. (1) \textbf{An in-depth analysis on guided decoding}:  we extend the previous discussion on outcome supervision and process supervision to the realm of guided decoding. We theoretically prove that \textit{outcome supervision for guided decoding is secretly a value model}.  (2) \textbf{A novel approach of OVM}: we propose Outcome Value Models for guided decoding that emphasize the potential correctness of the final answer rather than focusing solely on the current (partial) path’s correctness. Importantly,  OVM with \textit{outcome supervision} does not need costly step-level annotations typically required by \textit{process supervision}, making it more scalable. Moreover, it merely leverages the existing model and datasets without introducing external elements. (3) \textbf{Significance of OVM}:  OVM (7B) achieves state-of-the-art results in GSM8K among LLMs up to 13B parameters, even outperforming those using additional data, GPT-4, or code execution.







\section{Background}
\label{sec:background}


\begin{figure*}[h]
    \centering
    \subfigcapskip=-2pt
    \subfigure[\label{fig:reward_and_value}Reward and value]{
        \includegraphics[width=0.48\linewidth]{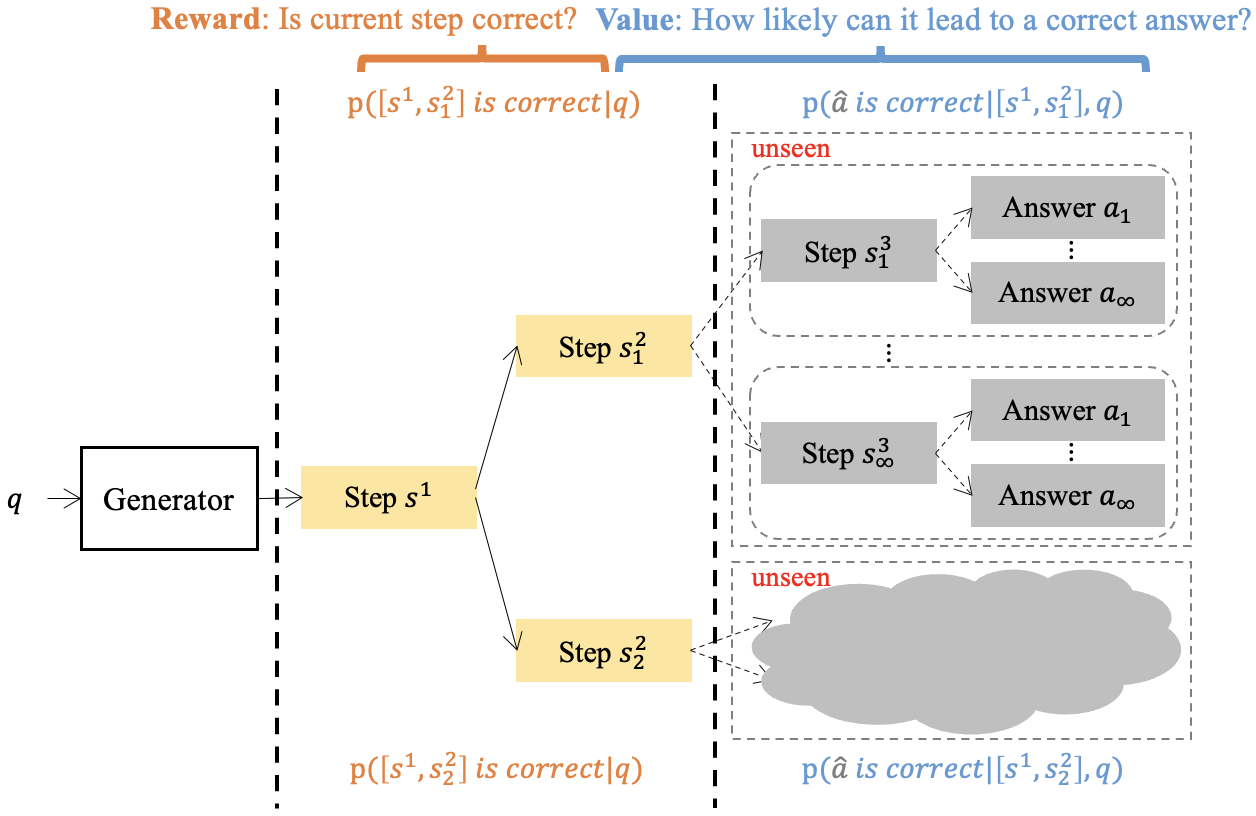}}
    \subfigure[\label{fig:ps_and_os}Outcome supervision and process supervision on training models to evaluate \textit{complete paths}]{
        \includegraphics[width=0.48\linewidth]{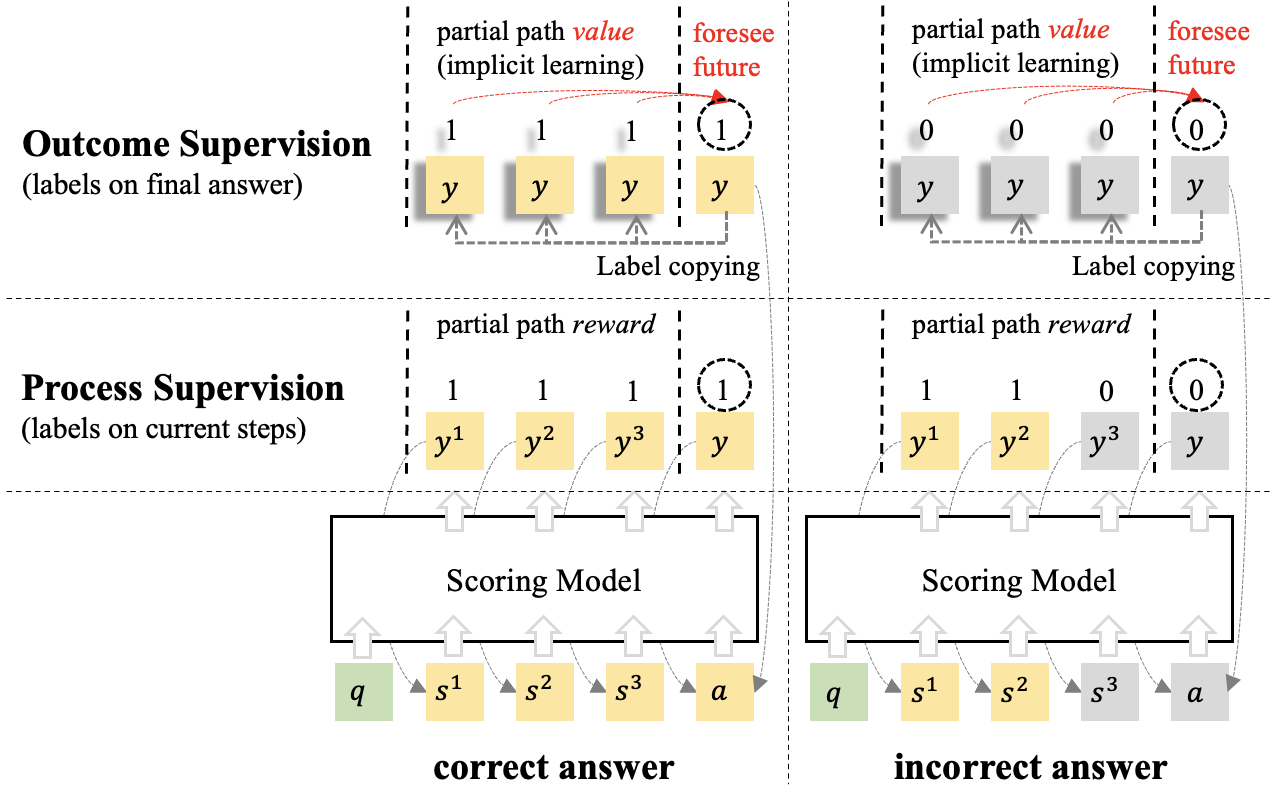}}
    \caption{(a): When evaluating partial paths (here for the first two steps), reward focuses on the current states, while value focuses on the unseen future outcomes. (b): Given a question $q$ and a solution path $[s^1,\cdots,s^m,a]$, models are trained to predict path correctness (circled output scalar on the last token). Outcome supervision replicates the final answer's correctness label across all steps (indicated by shaded labels), causing the model to implicitly learn to foresee the future, predicting values for partial paths. By contrast, process supervision details per-step correctness labels, causing the model to learn to predict step-level correctness, i.e. reward. Correct steps and answers are colored in \textcolor{yellow}{yellow} and incorrect ones in \textcolor{gray}{grey}.}
\end{figure*}

We  give the problem definition of mathematical reasoning and  guided decoding, as well as its dual paradigms (i.e., reward-based and value-based). The notations we used are summarized in Table~\ref{tab:notations}.


\subsection{Problem Defintion}
\label{sec:definition}

We first introduce the mathematical reasoning problem definition and then introduce our adopted paradigm.


\begin{definition*}
    \textbf{A mathematical reasoning question} \textit{$q$ requires a sequence of steps to be addressed, whose solution path is $S=[s^1, \dots, s^m, a]$, where $s^i$ represents the i-th step, $m$ is the number of steps, and $a$ is the final answer.} 
\end{definition*}




To alleviate the issue of potential error propagation from previous steps in a single solution chain, one approach is sampling multiple steps from the generator and filtering. 
This is called guided decoding, which incorporates a new evaluation criterion to select steps during model generation (\citealp{tot23}; \citealp{xie2023guided}; \citealp{grace23}; \citealp{alphazero-like23}).

\paragraph{Guided decoding}
Guided decoding intervenes in the generation process with a new evaluation criterion, in contrast to vanilla sampling which is solely based on the Language Model (LM)  probabilities. Specifically, for each step $t$, suppose the sampling size is $K$, the generator $\Phi$ produces a set of candidate paths $\mathbb{S}^{(1:t)} = \bigl\{S^{(1:t)}_k \bigl\}_{k=1}^{K}$ based on LM probabilities, where $S^{(1:t)}_k=[s^1_k,\dots,s^t_k]$ is the $k$-th partial path up to step t. Then, given an evaluation criterion $f(\cdot)$ that can score an incomplete path $S^{(1:t)}$, we select the top-scored paths with the beam size $b$ ($b < K$ for pruning), from which the generation continues 
\begin{equation*} \scriptsize
    \Bigg\{ S^{(1:t)}_k  \scalebox{1.8}{$\mid$} k \in \mathop{\mathrm{argtop}_b}\limits_{k=1,\cdots,K} f(S^{(1:t)}_k;q) \Bigg\}
\end{equation*}

This approach is primarily characterized by two categories of guiding criteria: reward and value, which are two concepts in reinforcement learning~\cite{sutton2018reinforcement}, see details in Section~\ref{sec:guided_models}.

\subsection{Reward-based and Value-based Guided Decoding}
\label{sec:guided_models}

In this subsection, we introduce  `reward'-based approaches  and `value'-based approaches for mathematical reasoning in Sec.~\ref{sec:reward} and \ref{sec:value}


\subsubsection{Reward-based Guided Decoding}
\label{sec:reward}

\textit{Reward}-based approaches (\citealp{xie2023guided}; \citealp{grace23}; \citealp{rap2023}), focusing on the generated steps, assess the correctness of the current steps in mathematical reasoning, i.e. $p(S^{(1:t)}\mbox{ is correct}|q)$.

\paragraph{Outcome supervision vs. process supervision}
In mathematical reasoning, reward models are well known in evaluating complete solution paths $p(S\mbox{ is correct}|q)$, also called ``verifiers'' (\citealp{gsm8k21}; \citealp{deepmind-process22}; \citealp{openai-process23}; \citealp{li2023stepv}; \citealp{grace23}). 
There are two supervision strategies to train a verifier, distinguished by the granularity of the supervision signals, we refer to Appendix~\ref{app:verifier} for training details.

\textbf{Outcome Supervision} simply focuses on the correctness of the final answer, at a coarser granularity. The trained model is called Outcome Reward Model (ORM) (\citealp{gsm8k21}; \citealp{deepmind-process22}; \citealp{openai-process23}).

 \textbf{Process Supervision} offers more fine-grained, step-wise labels of the solution path, providing per-step correctness. The trained model is called Process Reward Model (PRM) (\citealp{deepmind-process22}; \citealp{openai-process23}; \citealp{li2023making}).

Current research indicates that process supervision generally outperforms outcome supervision since the former adapts finer-grained supervision in \textit{verifying complete paths}~\citep{openai-process23}. 
However, in guided decoding that \textit{verifies incomplete paths},  typical reward models might overlook the current (incomplete) path's future implications, which will be further discussed in Section~\ref{sec:value}. 


\subsubsection{Value-based Guided Decoding}
\label{sec:value}

\textit{Value}-based approaches (\citealp{tot23}; \citealp{rap2023}; \citealp{alphazero-like23}) 
estimate the  expected future  \textit{rewards} when starting from a given state (i.e. the current incomplete reasoning path), which is future-oriented. This is contrast to the definition of \textit{rewards} that is determined only by the seen incomplete path and agnostic to the future path. 
As shown in Figure~\ref{fig:reward_and_value}, reward models assess paths in a backward direction (e.g., the correctness of seen steps) while  value models assess paths in a forward direction (e.g.,  the  potential correctness the final path with additional future unseen steps and the answer $\hat{a}$ ~\footnote{In reinforcement learning, value is defined as the expected cumulative reward it receives in the long run with a discount factor: $\sum_{j=1}^{m-t}\gamma^{j-1}R_{t+j}$. In our scenario, the discount factor is 1, all intermediate rewards $R_{t+1}R_{t+2}\cdots R_{m-1}$ are 0, and the final reward $R_m$ is 1 if the answer is correct otherwise 0. So the cumulative reward is either 1 or 0 dependent on the answer correctness. Therefore, the expected cumulative reward is exactly the probability of correct answers.
}).  
Interestingly,  we could term the \textit{value-based guided decoding} as ``\textit{planning}'' based on its nature of future orientation.

\section{Outcome Supervised Value Models for Guided Decoding}
\label{sec:methodology}

\subsection{Motivation}

\paragraph{Challenge of training value models}
Unlike labels of reward models can be annotated manually on a given (incomplete) reasoning path,  it is challenging to obtain ground truth of value models for each incomplete path during guided decoding. The reason is that it is computationally-heavy to calculate the expected rewards among all possible future paths starting from the seen (incomplete) path, especially the number of resulted sequences grows exponentially w.r.t the length of  reasoning paths.

\paragraph{Rationale behind outcome supervised guided decoding as a value model}
Therefore, the challenge in training value models lies in \textit{estimating or labeling the value of observed reasoning paths}. Recalling the types of supervision – either outcome or process – it's evident that process supervision is confined to paths already seen. However, outcome supervision appears to have the potential to assess the probable correctness of resulting final paths, starting from the current incomplete one. See this intuition in Figure~\ref{fig:ps_and_os}. Intriguingly, upon theoretical examination, we discover that \textbf{outcome supervision for guided decoding essentially acts as a value model}, as detailed in Sec.~\ref{sec:explanation}. This revelation has inspired the adoption of outcome-supervised value models specifically tailored for guided decoding.~\footnote{\citet{alphazero-like23} is a concurrent work with us for value model.}

\subsection{Outcome Supervision Leads to a Value Model for Guided Decoding}
\label{sec:explanation}
We show theoretically that, given binary labels of individual samples, outcome supervision implicitly estimates the global labels, or value, of the intermediate steps during the optimization process.

\begin{claim*}
    For an outcome-supervised model $f_{\bm{\theta}}(\cdot)$ parameterized by the optimal parameter $\bm{\theta}$,
    its score of $S^{(1:t)}$ is the approximated probability of it reaching a correct answer, i.e.,
    \begin{equation}
    f_{\bm{\theta}}(S^{(1:t)};q)=p(\hat{a}\mbox{ is correct}|S^{(1:t)},q)
    \end{equation}
\end{claim*}

\begin{proof}

Suppose for each question $q$, we have the generator producing $n$ solution paths $\{S_i\}_{i=1}^n$ with the corresponding answers $\{a_i\}_{i=1}^n$.
The label $y_i$ is 1 if $a_i$ is correct otherwise 0.
The mean squared error loss of outcome supervision is
\begin{equation} \small
l(S_i^{(1:t)},y_i; q)=\left(f_{\bm{\theta}}(S_i^{(1:t)}; q)-y_i\right)^2
\end{equation}

Given the training question set $\mathcal{Q}$, 
the overall objective is
\begin{equation} \small
L=\frac{1}{|\mathcal{Q}|}\sum_{q\in\mathcal{Q}}\frac{1}{n}\sum_{i=1}^n\sum_{t=1}^{m_i}\left(f_{\bm{\theta}}(S_i^{(1:t)}; q)-y_i\right)^2
\end{equation}

Denote $\mathrm{v}^{q}_{x}=f_{\bm{\theta}}(x;q)$, the partial derivation of $\mathrm{v}^{q}_{x}$ is
\begin{equation} \small
\nabla_{\mathrm{v}^{q}_{x}} L = \frac{1}{|\mathcal{Q}|}\frac{1}{n}\sum_{i=1}^n \sum_{t=1}^{m_i} 2\Gamma(S_i^{(1:t)}=x)(\mathrm{v}^{q}_{x} - y_i)
\end{equation}

Set $\nabla_{\mathrm{v}^{q}_{x}} L=0$, we can see
\begin{equation} \small
\mathrm{v}^{q}_{x} = \frac{\sum_{i=1}^n \sum_{t=1}^{m_i} \Gamma(S_i^{(1:t)}=x) y_i}{\sum_{i=1}^n \sum_{t=1}^{m_i} \Gamma(S_i^{(1:t)}=x)}
\end{equation}
which is $p(\hat{a}\mbox{ is correct}|x,q)$, whose estimation's precision depends on the sampling. Choose the model satisfying 
\begin{equation} \small
f_{\bm{\theta}}(x;q) = \frac{\sum_{i=1}^n \sum_{t=1}^{m_i} \Gamma(S_i^{(1:t)}=x) y_i}{\sum_{i=1}^n \sum_{t=1}^{m_i} \Gamma(S_i^{(1:t)}=x)}
\end{equation}
for all $x\in\{S_i^{(1:t)}\}$ is the optimal solution minimizing the loss function. Hence the optimal solution satisfies 
\begin{equation} \small
f_{\bm{\theta}}(S^{(1:t)};q)=p(\hat{a}\mbox{ is correct}|S^{(1:t)},q)
\end{equation}


\end{proof}

Therefore,
\begin{equation*}
\scriptsize
\begin{aligned}
f_{\bm{\theta}}(S;q)
&=p(a\text{ is correct}|S; q) \\
&=
    \begin{cases}
    \text{reward} & \text{when }[\overbrace{s^1, \dots, s^t, \dots, s^m,a}^{S}]\text{ i.e. }a \text{ is seen} \\\\
    \text{value} & \text{when }[\overbrace{s^1, \dots, s^t}^{S}, \dots, s^m,a]\text{ i.e. }a \text{ is unseen}
    \end{cases}
\end{aligned}
\end{equation*}

\paragraph{Intuitive Explanation}
This indirect method of probability estimation in outcome supervision simplifies the value model training process, which avoids the need for explicit step-level continual sampling and estimation for training labels. Instead, it leverages the binary correctness of individual samples as training labels, forcing the optimal solution to be the probability of being correct under mean square error, which is similar to the Monte Carlo method to estimate the expectation. 

\subsection{Rethinking Supervision for Guided Decoding: Outcome v.s. Process}
\label{sec:os_vs_ps}

With the above discussion, outcome supervision and process supervision can be different in the context of guided decoding. We claim that outcome supervision supersedes process supervision in this scenario for two reasons.

\paragraph{Outcome supervision is preferred due to its inherent future-guided orientation} 
For guided decoding, intuitively we should adopt a forward-looking approach that prioritizes the final answer's correctness over mere the current path's. 
This favors value models over typical reward models. Thus, outcome supervision, leading to value models, is preferred to process supervision that results in reward models, for partial path evaluation. 


\paragraph{Outcome supervision is labor-friendly without fine-grained annotations} 
In terms of future orientation, rewards can be modified to introduce such aspects, e.g. ``whether steps are correct and \textit{helpful} to the correct answer''. 
We acknowledge that such reward adjustments are useful for planning. 
However, annotating rewards at the step level is labor-intensive.
Furthermore, assessing steps' contribution to final answers, beyond mere correctness, increases the labor demands of reward labeling.
In contrast, outcome supervision only requires the final answer's correctness.


\section{Method}

\paragraph{Building a training set for OVM} 
Given a set of questions $\mathcal{Q}$ comprising $N$ training questions, we initially query the generator to produce $n$ solution paths $\mathbb{S} = \{S_1,\cdots,S_n\}$ for each question $q \in \mathcal{Q}$. This process yields $N\times n$ question-solution pairs. Subsequently, we determine the binary label for each question-solution pair by assessing the correctness of the final answer.

\paragraph{Training a value model with outcome supervision} 
The value model is implemented by adding a linear layer with a single bias parameter after the generator's final unembedding layer, separate from the generator~\citealp{gsm8k21}. The training objective is to minimize the mean squared error between the predicted value, based on the question and solution, and the binary label.

For comparative purposes, we implemented reward models trained through process supervision.~\footnote{Process supervision is only used in comparison, not in OVM training.} Detailed information can be found in Appendix~\ref{app:verifier}.

\paragraph{Inference - beam search with guided decoding}
During inference, we employ a beam search strategy guided by the OVM. Unlike the conventional beam search, which relies on token-level probability, our method is steered by the estimated values at each step. The algorithm is detailed in Algorithm~\ref{algo:beam_search}. 

\begin{algorithm}[h]
\small
\caption{\label{algo:beam_search}Value-Guided Beam Search}
\begin{algorithmic}[1]

\State $\textbf{Input:}$ Question $q$, Beam size $b$, Sampled steps per state $K$, Maximum step count $T$
\State $\textbf{Output:}$ Best solution sequence for $q$
\State $\textbf{Model:}$ Generator $\Phi$ and OVM $f$

\Procedure{ValueGuidedBeamSearch}{$q, b, K$}
    \State Initialize step sequences $\mathbb{S} \gets \{\}$
    \State Sample initial steps $\{s_1^1,\dots,s_K^1\}$
    \State Evaluate values $\{v_1^1,\cdots,v_K^1\}$ for each step
    \State Select top $b$ valued steps and add to $\mathbb{S}$
    \State $t \gets 1$
    \While{sequences in $\mathbb{S}$ are not complete and $t < T$}
        \State $\mathbb{S}_{\text{new}} \gets \{\}$
        \State $\mathcal{V} \gets \{\}$
        \For{each sequence $S^{(1:t)}$ in $\mathbb{S}$}
            \For{$i = 1$ to $K/b$}
                \State $S^{(1:t+1)}_i=\Phi(S^{(1:t)};q)$
                \State $v^{(1:t+1)}_i=f(S^{(1:t+1)}_i;q)$
                \State $\mathbb{S}_{\text{new}} \gets \mathbb{S}_{\text{new}}+S^{(1:{t+1})}_i$
                \State $\mathcal{V} \gets \mathcal{V}+v^{(1:t+1)}_i$
            \EndFor
        \EndFor

        \State $\mathbb{S}_{\text{new}} \gets $ top $b$ valued sequences from $(\mathbb{S}_{\text{new}},\mathcal{V})$
        \State $\mathbb{S} \gets \mathbb{S}_{\text{new}}$
        \State $t \gets t+1$
    \EndWhile
    \State \textbf{return} sequence with highest final value in $\mathbb{S}$
\EndProcedure

\end{algorithmic}
\end{algorithm}

\section{Experiment Results}
\label{sec:experiments}

\subsection{Experimental settings}

\paragraph{Benchmarks} We conduct experiments on two mathematical reasoning datasets, GSM8K~\citep{gsm8k21} and Game of 24~\citep{tot23}.

\paragraph{Baselines} We benchmark our method against leading models in GSM8K and the notable Tree-of-Thought in Game of 24~\citep{tot23}, as well as other guided decoding approaches. 
Additionally, we evaluate the efficacy of OVM planning against the vanilla sampling methods of our implemented generators, such as greedy search and post-processing of multiple solutions generated without guided decoding.

We conduct each inference experiment three times and present the average results along with their standard deviation. Given the variety of available beam sizes $b$ for each sampling size $K$, we simplify the reporting by only showcasing the best results from all possible beam sizes.~\footnote{For instance, the result for $K=20$ is the best one among $b\in(1,2,4,5,10)$.} Detailed results for different beam sizes can be found in Appendix~\ref{app:results}.

See the implementation details, including training and inference hyperparameters, in Appendix~\ref{app:implementation}.

\begin{table*}[t]
 \footnotesize
 \caption{Accuracy on GSM8K. In the third column, we mark models that use GPT for inference or are trained with GPT-generated data. Notably, we don't rely on GPT, data augmentation, and code execution (execute the complete code block outputting the final answer). SC denotes `Self-Consistency' and RM denotes `Reward Model'.}
  
 \label{tab:benchmark_gsm8k_others}
 \centering
 \begin{tabular}{lrccl}
 \toprule
 Model  & Size & GPT-3.5/4 & Data Augmentation  & Accuracy \\
 \midrule  
 \multicolumn{5}{l}{\textit{Open-Source Models without Code Execution}} \\
 MuggleMATH (\citealp{MuggleMATH23}) 
        & 7B   & \cmark   & \cmark             & 68.4\%    \\
 Arithmo-Mistral 
        & 7B   & \cmark   & \cmark             & 74.7\%    \\
 MetaMath-Mistral (\citealp{MetaMath23}) 
        & 7B   & \cmark   & \cmark             & 77.7\%    \\
 \hline
 MetaMath (\citealp{MetaMath23}) 
        & 13B  & \cmark   & \cmark             & 71.0\%    \\
 MuggleMATH (\citealp{MuggleMATH23}) 
        & 13B  & \cmark   & \cmark             & 74.0\%    \\
 \hline
 RFT (\citealp{RFT23})  
        & 70B  &          &                    & 64.8\%    \\
 WizardMath (\citealp{WizardMath23})
        & 70B   &         & \cmark             & 81.6\%    \\
 MuggleMATH (\citealp{MuggleMATH23}) 
        & 70B  & \cmark   & \cmark             & 82.3\%    \\
 MetaMath (\citealp{MetaMath23}) 
        & 70B  & \cmark   & \cmark             & 84.3\%    \\
 \textbf{Ours -- OVM} (Llama2-7B, K=100)  
        & \textit{7B}   &          &           & 73.7\% ± 0.4\%   \\
 \textbf{Ours -- OVM} (Mistral-7B, K=100) 
        & \textit{7B}   &          &           & \textbf{84.7}\% ± 0.3\%    \\
\midrule
 \multicolumn{5}{l}{\textit{Open-Source Models with Code Execution}} \\
 ToRA-Code (\citealp{ToRA-Code23}) 
        & 7B   & \cmark   & \cmark             & 72.6\%    \\
 ToRA-Code (\citealp{ToRA-Code23}) 
        & 13B  & \cmark   & \cmark             & 75.8\%    \\
 \makecell[l]{ToRA-Code (SC, K=50) (\citealp{ToRA-Code23}) } & 34B  & \cmark   & \cmark             & 85.1\%    \\
 ToRA (\citealp{ToRA-Code23}) 
        & 70B  & \cmark   & \cmark             & 84.3\%    \\
 \makecell[l]{ToRA (SC, K=50) (\citealp{ToRA-Code23})}  
        & 70B  & \cmark   & \cmark             & 88.3\%    \\
\midrule
  \multicolumn{5}{l}{\textit{Closed-Source Models}} \\
 \makecell[l]{PaLM (SC, K=32)  (\citealp{self-improve22})}  
        & 540B &          &                    & 82.1\%    \\
 \makecell[l]{DeepMind's+RM Verification (K=96)  (\citealp{deepmind-process22})}  
        & 70B  &          &                    & 87.3\%    \\
 GPT-4 (\citealp{GPT-4})  
        & -    & \cmark  &                    & 87.1\%    \\
 \makecell[l]{GPT-4 Code+Self-Verification (K=5) (\citealp{CSV23})}  
        & -   & \cmark   &                    & 97.0\%    \\


 \bottomrule   

 \end{tabular}

\end{table*}

\subsection{Overall Performance}

\paragraph{Benchmarking against current state-of-the-art approaches} 
The OVM performance in GSM8K and Game of 24 is detailed in Table~\ref{tab:benchmark_gsm8k_others} and Table~\ref{tab:benchmark_game24_others}, respectively. 
Notably, our Mistral-based 7B model surpasses all models under 70B in GSM8K. In the 7B category, excluding Mistral-based models, our Llama2-based 7B model achieves the highest performance. 
In the Game of 24, OVM planning significantly improves Llama2-7B's accuracy, increasing its accuracy from 11\% to a remarkable 78.7\% with 20 sampled solution paths.

\begin{table}[H]
 \small
 \caption{Accuracy on Game of 24. GPT-4's accuracy is from \citet{tot23}, and K of ToT is estimated from Figure 3 in their paper.}
  
 \label{tab:benchmark_game24_others}
 \centering
 \begin{tabular}{l|l}
 \toprule
                                        & Accuracy \\
 \midrule	
  GPT-4 CoT & 4.0\% \\
  GPT-4 SC (K=100) & 9.0\% \\
  GPT-4 ToT (K=60) & 74.0\% \\
  \midrule
  Fine-tuned Llama2-7B                  & 11.0\% \\
  Fine-tuned Llama2-7B SC (K=100)       & 11.7\% ± 1.3\% \\
  \textbf{Ours -- OVM} (Llama2-7B, K=20)         & 78.7\% ± 1.7\%\\
  \textbf{Ours -- OVM} (Llama2-7B, K=100)         & 98.3\% ± 1.2\%\\
 
 \bottomrule   
 \end{tabular}

\end{table}

\paragraph{Benchmarking against guided decoding approaches}
Table~\ref{tab:benchmark_gsm8k_guided} shows that OVM excels over most guided decoding approaches, with the exception of the GPT-based method. Remarkably, OVM achieves comparable results to the GPT-based method despite its smaller size (7B compared to 175B) and fewer sampled paths ($K=10$ versus $K=80$). Significantly, OVM improves the previous value-based SOTA by 18 absolute percentage points (from 63.2\% of CoRe to 81.2\%), eliciting the power of value-based methods.

\paragraph{Benchmarking against vanilla sampling baselines} Table~\ref{tab:benchmark_all_self} shows OVM planning generally outperforms ORM post-selection with the same number of sampled paths. 
An exception occurs with Mistral-7B at K=100 in GSM8K, where the gap between OVM and ORM approaches appears to reach saturation. 
As shown in Figure~\ref{fig:all_k_acc_proportion}, in both GSM8K and Game of 24, the accuracy improves with larger sampling sizes. The gap between ORM and OVM decreases as more paths are sampled. 

Notably, training an OVM merely reuses existing models and datasets, generating training solutions and labels internally. This approach outperforms those needing extra resources like code execution or data augmentation. Moreover, its compatibility with these techniques suggests the potential for further improved performance when used together.

\begin{table*}[t]
 \footnotesize
 \caption{Accuracy on GSM8K comparing with guided decoding approaches. `RM' denotes `Reward Model', and `VM' denotes `Value Model'. `finetuned' means the generator is tuned on the training dataset.}
  
 \label{tab:benchmark_gsm8k_guided}
 \centering
 \begin{tabular}{llrcp{3.2cm}l}
 \toprule
 Model  & Backbone & Setting & K & Type  & Accuracy \\
\midrule
 \multicolumn{5}{l}{\textit{Reward-based}} \\
 GRACE (\citealp{grace23})      
        & Llama-7B  & finetuned & 20   &  RM             & 30.9\%    \\
 GRACE (\citealp{grace23})      
        & T5-770M  & finetuned & 20   &  RM             & 36.3\%    \\
 
 SelfEval (\citealp{xie2023guided}) 
        & GPT3.5-Codex-175B & prompting & 80 & Prompting           & 85.5\%    \\
\midrule
\multicolumn{5}{l}{\textit{Value-based}} \\
 RAP (\citealp{rap2023})      
        & Llama-33B   & prompting & 10   &  Simulation              & 51.6\%    \\
 CoRe (\citealp{grace23})      
        & GPT-J-12B   & finetuned & 40   & Simulation              & 63.2\%    \\
 \citet{alphazero-like23}   
        & Llama2-7B   & finetuned &  10  &  VM             & 52.2\% ± 0.9\%    \\
 \citet{alphazero-like23}   
        & Llama2-7B   & finetuned &  50  &  Simulation+Aggregation             & 59.4\%    \\
 \textbf{Ours -- OVM} (Llama2-7B)  
        & Llama2-7B    & finetuned & 10         & VM          & 66.5\% ± 0.2\%   \\
 \textbf{Ours -- OVM} (Mistral-7B) 
        & Mistral-7B    & finetuned & 10         & VM          & \textbf{81.2}\% ± 0.6\% \\

 \bottomrule   

 \end{tabular}

\end{table*}

\begin{table*}[t]
 \small
 \caption{Accuracy on GSM8K and Game of 24. Results averaged over 3 runs are reported. K denotes sampling size.}
  
 \label{tab:benchmark_all_self}
 \centering
 \begin{tabular}{lll|ll|l}
 \toprule
 \multicolumn{3}{c|}{Method} & \multicolumn{2}{c|}{GSM8K} & Game of 24 \\
                        & &  & Llama2-7B & Mistral-7B   & Llama2-7B \\
 \midrule		
 
 \multirow{5}{*}{Vanilla Sampling} 
                                 & \multicolumn{2}{c|}{Greedy}                                       & 38.6\%  & 58.4\%   & 11\% \\
                                 \cline{2-6}
                                 & \multirow{2}{*}{SC}     & K=20  & 53.3\% ± 0.3\%      & 70.2\% ± 0.7\% & 10.3\% ± 1.7\% \\
                                 &                         & K=100 & 57.4\% ± 0.8\%     & 72.6\% ± 0.2\% & 11.7\% ± 1.3\% \\  
                                 \cline{2-6}
                                 & \multirow{2}{*}{ORM}    & K=20  & 65.5\% ± 0.7\%      & 81.8\% ± 0.2\% & 65.3\% ± 5.3\% \\
                                 &                         & K=100 & 71.9\% ± 0.6\%      & \textbf{84.7}\% ± 0.4\% & 95.3\% ± 0.5\% \\
                                 \cline{2-6}
                                 & \multirow{2}{*}{PRM}    & K=20  & 66.4\% ± 0.5\%      & - & 60.3\% ± 4.2\% \\                                 
                                 &                         & K=100 & 70.8\% ± 0.7\%      & - & 93.3\% ± 0.9\% \\
 
 \midrule
 \multirow{2}{*}{Planning}       & \multirow{2}{*}{OVM}    & K=20  & 69.0\% ± 0.3\%       & 82.6\% ± 0.1\% & 78.7\% ± 1.7\% \\
                                 &                         & K=100 & 73.8\% ± 0.4\%       & \textbf{84.7}\% ± 0.3\% & \textbf{98.3}\% ± 1.2\% \\
 \bottomrule   
 \end{tabular}

\end{table*}

\section{Analysis and Discussion}

This section seeks to answer the following two \textbf{R}esearch \textbf{Q}uestions (RQs).
\begin{researchquestion}
    Can OVM plan?
\end{researchquestion}

\begin{researchquestion}
    How is outcome supervision compared to process supervision for guided decoding?
\end{researchquestion}



\paragraph{Evaluation with correct answer proportion}
To assess planning effectiveness, we analyze the proportion of sampled solution paths yielding correct answers in the final sampling stage, immediately preceding the final solution selection. This offers insights into guided decoding's efficiency in steering towards correct answers.

\subsection{RQ1: Can OVM plan?}
\label{sec:rq1}

We use ``vanilla sampling'' as the baseline for comparison, which relies on random solution sampling based solely on LM probabilities, without guiding.

\paragraph{OVM is an effective planner guiding to the correct answers}
The result is shown in Figure~\ref{fig:all_k_acc_proportion}. 
Notably, in GSM8K, less than 35\% of the generator's randomly sampled solution paths are correct, and this proportion increases to over 65\% with OVM planning. Similarly, in the Game of 24, OVM planning significantly boosts the correct answer proportion from approximately just 10\% in vanilla sampling to an impressive 80\%. 
Additionally, in vanilla sampling, the proportion of correct solutions remains consistent across various sampling sizes. In contrast, OVM planning demonstrates improved benefits with increased sampling sizes, up to a point of saturation.


\subsection{RQ2: How is outcome supervision compared to process supervision for guided decoding?}
\label{sec:comparison}
We further compare the performance of reward models, trained under process supervision~\footnote{The training datasets and hyperparameters for the reward models are identical to those of OVM}, with OVM in guiding decoding for GSM8K and Game of 24. Due to resource constraints, we only conducted the experiments on Llama2-7B.

We investigate both typical and modified future-oriented rewards in our study. 
The former rewards steps (i.e. labeled as 1) for logical correctness, while the latter rewards steps that are not only correct but also contribute to the correct final answer. 
We refer to the model trained with this enhanced reward scheme as PRM-O, denoting its implicit consideration of future Outcomes.
See the details for per-step correctness annotations in Appendix~\ref{app:process_labels}.

\paragraph{Comparison in Game of 24} 
Figure~\ref{fig:game24_ep_proportion} demonstrates the evolving trends across training epochs, illustrating that both OVM and PRM-O effectively guide towards correct answers, in contrast to PRM's failure, highlighting the importance of anticipating outcomes.~\footnote{Two more training epochs are conducted for better visualization.} 
Notably, PRM-O shows faster convergence than OVM, but OVM eventually reaches a performance comparable to PRM-O. This indicates that outcome supervision, relying only on final answer correctness, may suffice for models to learn outcome evaluation, while more detailed step-level signals can accelerate this process.

\begin{figure}[H]
    \centering
    \includegraphics[width=1\linewidth]{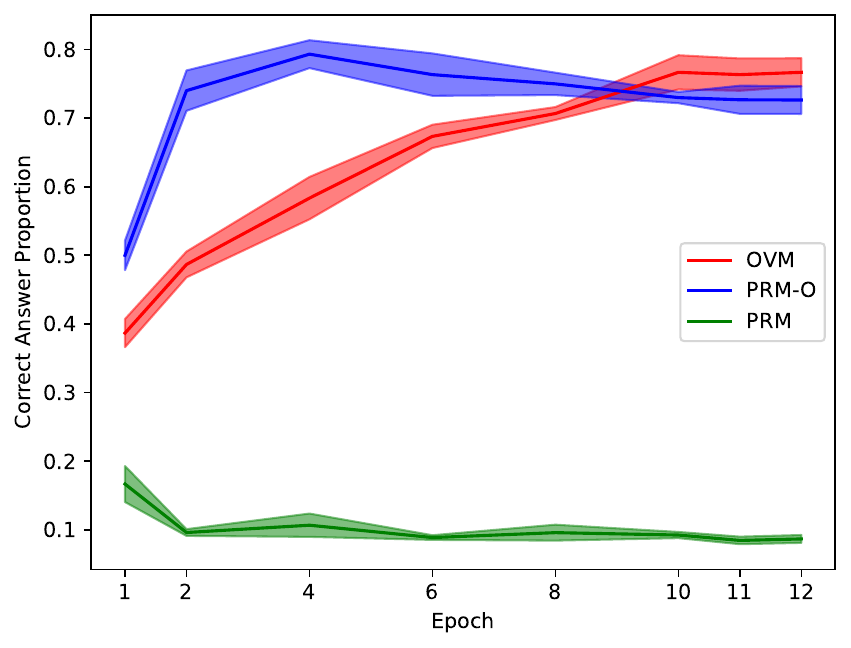}
    \caption{Comparison among OVM, PRM, and PRM-O in correct answer proportion in Game of 24 (K=20)}
    \label{fig:game24_ep_proportion}
\end{figure}


\paragraph{Comparison in GSM8K} 
See the results in Table~\ref{tab:comparison}. PRM-O outperforms PRM, consistently favoring anticipating outcomes. However, OVM and PRM show similar performance levels. We analyze the reason behind OVM's lack of superior performance over PRM as follows.

\begin{table}[H]
 \caption{Correct answer proportion in GSM8K in comparison among OVM, PRM, and PRM-O.}
  
 \label{tab:comparison}
 \small
 \centering
 \begin{tabular}{l|ll}
 \toprule
              & K=20                 & K=100 \\
 \midrule
 OVM          & 65.8\% ± 0.6\%     & 68.9\% ± 0.2\% \\
 PRM          & 65.9\% ± 0.6\%     & 69.1\% ± 0.3\% \\
 PRM-O        & 67.4\% ± 0.6\%     & 70.4\% ± 0.2\% \\
 \bottomrule   
 \end{tabular}

\end{table}

\paragraph{Analysis of the difference between Game of 24 and GSM8K} 
Distinct patterns emerge in comparing outcome supervision versus process supervision across Game of 24 and GSM8K, likely due to data specificity and data efficiency.

\textbf{Data specificity} concerns how well ``correctness'' aligns with ``helpfulness''. Our analysis shows a stark contrast in the consistency between PRM labels (logical correctness only) and PRM-O labels (emphasizing contribution to the correct answer) in Game of 24 and GSM8K, which are 56.9\% and 98.6\% respectively. 
This suggests that in Game of 24, logical correctness does not reliably predict answer success, making PRM vulnerable in such a scenario. Conversely, OVM/PRM-O, with its emphasis on the helpfulness towards the final answer, appears more robust. In GSM8K, where logically correct steps typically lead to correct outcomes (98.6\% of cases), PRM is nearly as effective as OVM in finding correct answers. 

\textbf{Data efficiency} considers the dataset size relative to task complexity. In scenarios like GSM8K, where the dataset is small for the task's complexity, PRM/PRM-O might offer more efficiency through detailed step-by-step supervision. However, in simpler tasks with adequately large datasets, such as the Game of 24, while fine-grained supervision might speed up training, it doesn't necessarily translate to better performance.

Overall, when considering both the performance and annotation costs (see the statistics in Appendix~\ref{app:annotate_cost}), outcome supervision demonstrates superior utility across various settings: it reaches competitive and even better performance than process supervision, with significantly reduced demands on annotation efforts.

\section{Related Works}

\paragraph{Complete path verification in mathematical reasoning} 
Mathematical reasoning presents significant challenges in arithmetic computation and complex, multi-step reasoning. The complexity of such tasks arises from the ease of making mistakes at each step, which can influence subsequent steps and final answers. In such scenarios, Verification has gained popularity as a means of improving accuracy by prioritizing the most plausible solutions among multiple alternatives (\citealp{generate_rank21}; \citealp{gsm8k21}; \citealp{back-verification22}; \citealp{CSV23}). A common implementation of verification involves training a specialized model to predict the correctness of complete solutions, which is called the verifier (\citealp{gsm8k21}; \citealp{deepmind-process22}; \citealp{li2023stepv}; \citealp{grace23}). In training verifiers, a debate exists between outcome-based and process-based supervision, with recent trends favouring process supervision (\citealp{deepmind-process22}; \citealp{openai-process23}). In this paper, we explore the potential of outcome supervision in planning.

\paragraph{Guided decoding in multi-step problem solving} 
Compared to selecting from the completed paths, it is more efficient to guide the model decoding in the middle of the process to filter harmful or less helpful steps for multi-step problem-solving. There are mainly two types of evaluation criteria for intermediate steps: reward-based (past-oriented) and value-based (future-oriented). Reward-based methods assess the intermediate steps according to their correctness or other characteristics of the steps already taken (\citealp{grace23}; \citealp{rap2023}; \citealp{xie2023guided}). In contrast, value-based methods evaluate the intermediate steps based on the potential outcomes in the unseen future (\citealp{tot23}; \citealp{rap2023}; \citealp{CoRe23}; \citealp{alphazero-like23}). The previous value-based approaches evaluate the future potential through simulation, utilizing Monte Carlo Tree Search (\citealp{rap2023}; \citealp{CoRe23}). ToT simplifies the simulation process using heuristics aided by GPT-4 in the Game of 24 (\citealp{tot23}). However, creating heuristics for more complex and realistic mathematical datasets, such as GSM8K, poses significant challenges. In this paper, we explore developing a specialized model to predict values on the fly without complex simulation.




\section{Conclusion}
In conclusion, this paper presents a novel approach in verifying intermediate steps and guiding model generation. This is achieved through the introduction of the Outcome-based Value Model (OVM), which employs outcome supervision in training a value model for intermediate steps. Both theoretical and empirical evidence highlight the effectiveness of outcome supervision for value estimation in planning, offer a method that is more efficient and effective than process supervision, which results in a reward-based model. The OVM, requiring no costly step-level annotations and fewer sampled paths, demonstrates superior performance in complex multi-step reasoning tasks, as evidenced by its state-of-the-art results on GSM8K and impressive success rate improvement in the Game of 24.



\section*{Limitations}
Guided decoding, while effective, introduces a limitation: it introduces an additional model to aid the generator decoding, which imposes a more substantial demand on memory resources and decelerates the inference process. This poses a significant challenge in some real-world applications where rapid response is crucial. 
Besides, our study does not delve into the costs associated with training a sufficiently accurate value model. 
While process supervision may enable the training of a reward model with a small dataset, outcome supervision could necessitate a considerably larger dataset for the effective training of a value model. This raises concerns about the scalability of such a system. 
Additionally, the generalization capability of the value model remains unexplored in our research. This omission leaves unanswered questions regarding the model's adaptability and performance consistency across diverse or unforeseen scenarios.

\section*{Acknowledgement}
This work is supported by the Shenzhen Science and Technology Program (JCYJ20220818103001002), Shenzhen Doctoral Startup Funding (RCBS20221008093330065), and Tianyuan Fund for Mathematics of National Natural Science Foundation of China (NSFC) (12326608).


\bibliography{anthology,custom}

\newpage
\appendix


\begin{table}[t]
\scriptsize
\centering
\resizebox{0.48\textwidth}{!}{
\begin{tabular}{cl}
\hline
\textbf{Notation} & \textbf{Description} \\ \hline
$q$ & Mathematical reasoning question requiring a sequence of steps \\
$\mathcal{Q}$ & A question set \\
$S$ & Solution path for a question, $S=[s^1, \dots, s^m,a]$ \\
$s^i$ & The i-th step in a solution path \\
$a$ & Final answer in a solution path \\
$m$ & Number of steps in a solution path \\
$y$ & A binary label, either 1 or 0, indicating the correctness of $a$ \\
$S^{(1:t)}$ & Partial solution path up to step $t$, $S=[s^1, \dots, s^t]$ \\
$\mathbb{S}^{(1:t)}$ & Set of candidate partial paths $\mathbb{S}^{(1:t)} = \bigl\{S^{(1:t)}_k \bigl\}_{k=1}^{K}$ \\
$K$ & Sampling size for candidates when inference \\
$b$ & Beam size for selecting top-scored candidates \\
$n$ & Number of sampled paths per question for training value models \\
$\Phi$ & The language model as the generator  \\
$f$ & A scoring model that maps a partial path to a number \\
$v^{(1:t)}$ & Value (a number) of a partial path up to step $t$ \\
$\bm{\theta}$ & Model parameter \\
PRM & Process Reward Model, trained with process supervision \\
ORM & Outcome Reward Model, trained with outcome supervision \\
OVM & Outcome Value Model, trained with outcome supervision \\
\bottomrule
\end{tabular}
}
\caption{Summary of Notations Used in the paper}
\label{tab:notations}
\end{table}

\begin{table*}[h]
    \small
    \centering
    \begin{tabular}{l|ll}
    \toprule
                        & Complete path evaluation & Partial path evaluation \\
    \midrule
    Process supervision & Reward                   & Reward \\
    Outcome supervision & Reward                   & \textcolor{red}{Value} \\
    \bottomrule
    \end{tabular}
    \caption{\label{tab:prm_ovm} Types of scores predicted by process- or outcome-supervised models on the complete path and partial path, respectively. When evaluating partial paths, the predicted scores of outcome-supervised models are values.}
\end{table*}

\section{Training Strategies}
\label{app:verifier}

\subsection{Outcome supervision for OVM}

We train OVM with outcome supervision.

\paragraph{Training labels in outcome supervision} In outcome supervision, each question-solution pair only requires a single binary label $y_o\in\{0,1\}$, indicating whether the final answer $a$ is correct or not. In practice, this label is expanded into a consistent vector, $\bm{y}^o=[y_o,\dots,y_o]$, matching the length of the token sequence to enhance the robustness of training procedures (\citealp{gsm8k21}).

\paragraph{Training objective in outcome supervision} Given the training data ($q$, $S$, $\bm{y}^o$), the mean squared error loss is calculated as
$$l(S,\bm{y}^o;q)=\left||f(q;S)-\bm{y}^o\right||^2$$
Additionally, the model is jointly trained with language modeling loss unweighted, following \citet{gsm8k21}.

\subsection{Process supervision for PRM}
We train reward models (PRM and PRM-O) for comparison with process supervision. 

\paragraph{Training labels in process supervision} In process supervision, each question-solution pair requires a vector of labels, $[y_1, \dots, y_m]$, corresponding to the number of steps involved, denoted by $m$. Each element within this vector indicates the correctness of its respective step. In practice, this vector is expanded to align with the token sequence length by attributing the identical label, $y_i$, across all tokens within the i-th step. This results in the final label vector $\bm{y}^p=[y_1,\dots,y_1,\dots,y_m,\dots,y_m]$.

\paragraph{Training objective in process supervision} Process supervision shares the same training objective as outcome supervision, but differs in training labels.
$$l(S,\bm{y}^p;q)=\left||f(q;S)-\bm{y}^p\right||^2$$

\section{Implementation Details}
\label{app:implementation}

\paragraph{Training generators} We use the newline character as the marker for the end of each step. In GSM8K, we fine-tuned Mistral-7B and Llama2-7B on the training set. Given that GSM8K provides calculation annotations, our models were also trained to utilize calculators. In Game of 24, we fine-tuned Llama2-7B on problem indices 1-900 with enumerated solution paths. For both datasets, the fine-tuning was carried out for 2 epochs, with a batch size of 128. We set the maximum learning rate to 1e-5, using a linear scheduler with the AdamW optimizer. We implement FlashAttention for Llama2-7B (\citealp{flashattention22}, \citealp{flashattention23}).



\paragraph{Building training dataset for OVM}
Given the training question set, we first sample 100 solution paths for each question from the generator. The decoding temperature is 0.7, top k is 50, top p is 1.0 and the maximum new tokens is 400. Then, we detect the answer correctness for each sample. In GSM8K, the answer correctness is determined by exact string matching to the ground truth since all the answers are integers. In Game of 24, the answer correctness is based on the validity of the equation equating to 24 and the singular usage of input numbers, following~\citep{tot23}. This allows us to produce numerous OVM training samples using just question-answer pairs, without path annotations, resulting in 747,300 training samples for GSM8K and 90,000 for Game of 24.

\paragraph{Training OVMs/ORMs}~\footnote{Since we train the value model with outcome supervision, the objective originally intended for ORM, the same model is simultaneously used as OVM for partial paths and ORM for complete paths in this paper.}
OVMs were initialized from the corresponding generator checkpoints.  In GSM8K, OVM was trained for 1 epoch with a batch size of 512. In Game of 24, OVM is trained for 10 epochs with a batch size of 128, due to its smaller training set. The optimizer was AdamW and the maximum learning rate was set to 1e-5 for Llama2-7B and 2e-6 for Mistral-7B respectively, following a linear scheduler.

\paragraph{Training PRM and PRM-O} Same as OVM, PRM and PRM-O were initialized from the corresponding generator checkpoints. The training dataset is the same set of question-solution pairs as in OVM, but details per-step correctness as training labels. See the annotation details in Appendix~\ref{app:process_labels}. All the training hyperparameters are consistent to OVM's.

\paragraph{Value-guided beam search}
The decoding temperature is 0.7, top k is 50, and top p is 1.0. We set the maximum new token length as 400 and the maximum number of steps as 10. In Game of 24, the generator produces more duplicated outputs due to the small output space. During the beam search process, we give priority to non-duplicate sequences for selection.

\section{Process Label Annotation}
\label{app:process_labels}

\subsection{Annotation protocol}
\paragraph{Game of 24} 
We derive process labels by checking the syntax and calculation, and matching to all possible correct solutions, enumerated by rules. Specifically, for PRM training (logical correctness only), steps are labeled as 1 when the steps are logically correct in rules, i.e. the calculation is correct and each used number is given and only used in once. For PRM-O training (logical correctness and helpfulness), steps are labeled as 1 when they correspond to any of the enumerated feasible correct solution paths.

\paragraph{GSM8K}
We query GPT-4 to annotate process labels without references. GPT-4 is asked to classify each step into ``correct'', ``incorrect'', or ``unnecessary''. The used prompt is shown as follows:

\vspace{2mm}
\hrule
\begin{quote}
[Question] \\
\texttt{question} \\

[Correct Answer] \\
\texttt{answer} \\

[Solution] \\
\texttt{solution path} \\

[System] \\
You are an expert math examiner. Review the student's solution and mark each step as correct only if it's based on accurate premises and helps solve the problem. Mark it as "unnecessary" when it is logically valid but doesn't help. \\
Please mark with "[Conclusion]" and summary all your judgements in the format of "Step i is correct/incorrect/unnecessary".

\end{quote}
\hrule
\vspace{2mm}

We label ``correct'' steps as 1 and ``incorrect'' steps as 0. For PRM training, ``unnecessary'' steps are labeled as 1, while for PRM-O training, they are labeled as 0.

\subsection{Consistency evaluation for GPT-4 labeling in GSM8K}
We conduct a consistency evaluation of GPT-4 labeling compared to human labeling on a small set. 

\paragraph{Evaluation set construction}
To ensure coverage of this set across paths of different lengths, we randomly select two solutions from each length set, including one with the correct final answer and one with an incorrect answer. For instance, we sample a correct solution and an incorrect solution from both the 1-step path set and the 2-step path set. Additionally, we apply the same sampling procedure to the sets classified by the step length of reference solutions. Finally, we get 116 question-solution pairs.

\paragraph{Human evaluation}
We hire three master-level students to annotate those examples as ``correct'', ``incorrect'', or ``unnecessary''. 

\paragraph{Consistency analysis}
The agreement rates between GPT-4 labels and human labels are shown in Table~\ref{tab:agreement}. This indicates GPT-4 can provide process labels with high consistency to humans.

\begin{table}[h]
    \small
    \centering
    \begin{tabular}{l|llll}
    \toprule
                        & Human 1 & Human 2 & Human 3 & GPT-4 \\
    \midrule
    Human 1             & -       & 0.89    & 0.88    & 0.87 \\
    Human 2             & 0.89    & -       & 0.91    & 0.86 \\
    Human 3             & 0.88    & 0.91    & -       & 0.86 \\
    GPT-4               & 0.87    & 0.86    & 0.86    & -    \\
    \bottomrule
    \end{tabular}
    \caption{\label{tab:agreement} Agreement rates on 116 samples between GPT-4 labeling and human labeling}
\end{table}

\subsection{Annotation cost}
\label{app:annotate_cost}
PRM and PRM-O incur a considerably higher annotation cost compared to OVM due to the need for detailed per-step correctness assessments. With $N$ questions and $n$ sampled solution paths per question, and an average step count of $m$ for these paths, the annotation cost for process supervision scales as $O(Nnm)$. In contrast, the annotation cost for outcome supervision is $O(Nn)$, requiring only the final answer's correctness for each question-solution pair. Specifically, see the data statistics in Table~\ref{tab:annotate_cost}. 

\paragraph{Cost comparison in Game of 24}
In Game of 24, the annotation cost for PRM and PRM-O is four times higher than that of OVM, corresponding to the average number of steps in solution paths. Clearly, the longer the solution path, the greater the annotation cost disparity between process supervision and outcome supervision.

\paragraph{Cost comparison in GSM8K} 
In GSM8K, where each question has a unique answer, the final answer correctness for each sampled solution can be derived by comparing the final answer to the ground truth. This process significantly lowers the annotation cost for outcome supervision to $O(N)$, compared to $O(Nn(m-1)+N)$ for process supervision. Consequently, for OVM, 7,473 annotations are needed (equivalent to the number of questions), whereas PRM and PRM-O require 2,619,923 annotations — 350 times more than OVM. 

These comparisons underscore OVM's lower annotation cost and better scalability.

\begin{table*}[h]
    \scriptsize
    \centering
    \begin{tabular}{l|l|ll|ll|ll|ll}
    \toprule
    & \textbf{\#Questions} & \multicolumn{2}{c|}{\textbf{\#Solutions}} & \multicolumn{2}{c|}{\textbf{\#Steps}} & \multicolumn{2}{c|}{\textbf{Cost in outcome supervision}} & \multicolumn{2}{c}{\textbf{Cost in process supervision}} \\ 
    & all        & per question & all & per solution & all & labels & annotations & labels & annotations \\ 
    \midrule
    Game of 24 & 900              & 100                        & 90,000     & 4.01                   & 360,914 & 90,000 & 90,000 & 360,914 & 360,914  \\ 
    GSM8K & 7,473              & 100                        & 747,300     & 4.50                   & 3,359,750 & 747,300 & 7,473 & 3,359,750 & 2,619,923 \\ 
    \bottomrule
    \end{tabular}
    \caption{\label{tab:annotate_cost} Label and annotation statistics of outcome supervision and process supervision in GSM8K and Game of 24.}
\end{table*}

\section{Detailed Experiment Results and Hyperparameter Analysis}
\label{app:results}

There are two critical hyperparameters in value-guided beam search: sampling size and beam size. We present the detailed results across various sampling sizes and beam sizes in Table~\ref{tab:gsm8k_results_llama}, Table~\ref{tab:gsm8k_results_mistral} and Table~\ref{tab:game24_results}.

\subsection{Impact of beam sizes on OVM planning}
In this section, we mainly explore the impact of beam size choices on OVM performance. Notably, there is one special case: when the beam size is equal to the sampling size, the approach functions as vanilla sampling rather than guided decoding, as it omits any intermediate selection or pruning.

\paragraph{Inference cost is consistent across various beam sizes} 
Regardless of the beam size, given a fixed sampling size, the inference cost typically remains unchanged. This uniformity arises because the generator produces a consistent number of next steps (i.e., the sampling size) at each level of the tree, leading to stable peak memory usage and inference time which is primarily influenced by the generation phase, not by beam selection or data storage.

\paragraph{The impact of beam size on OVM effectiveness} We can observe from the tables that

\textbf{(1) A relatively large beam size enhances the accuracy}. 
In GSM8K (Table~\ref{tab:gsm8k_results_llama} and Table~\ref{tab:gsm8k_results_mistral}), accuracy improves with an increase in beam size, but the proportion of correct answers first rises then falls. In Game of 24 (Table~\ref{tab:game24_results}), accuracy initially increases before declining, while the correct answer proportion consistently decreases. These observations imply that a larger beam size can positively impact accuracy up to a point, beyond which it may become detrimental. We attribute the pattern to (1) an initial reduction in error propagation risk with increasing beam size, leading to higher accuracy, and (2) an excessively large beam size potentially introducing incorrect solutions, as indicated by the drop in correct answer proportion, thereby increasing the risk of false positives.

\textbf{(2) OVM's superiority over vanilla sampling is robust}. 
As shown in Table~\ref{tab:gsm8k_results_llama}-Table~\ref{tab:game24_results}, OVM demonstrates a robust and consistently superior planning capability compared to vanilla sampling, as evidenced by a higher proportion of correct answers across all beam sizes, including when the beam size is 1 (losing the advantage of error propagation), in both GSM8K and Game of 24. Additionally, OVM achieves better accuracy across a range of moderate beam sizes, rather than limited to specific settings, indicating the effectiveness of OVM planning over vanilla sampling is not a result of cherry pick. For example, any beam size of 4 or greater improves accuracy over vanilla sampling for K=20 in GSM8K (Table~\ref{tab:gsm8k_results_llama} and Table~\ref{tab:gsm8k_results_mistral}). Similarly, for K=100 in Game of 24, all beam sizes between 10 and 25 surpass the performance of vanilla sampling (Table~\ref{tab:game24_results}). The small standard variation further underscores the reliability of these improvements.

\subsection{Comparison between outcome supervision and process supervision}

\paragraph{Comparison between two supervision strategies in guided decoding across various beam sizes} 
When evaluating the performance across a spectrum of beam sizes beyond just the peak performance, the analysis consistently shows that outcome supervision is competitive and even better than process supervision in terms of effectiveness. Specifically,

\textbf{(1) Outcome supervision excels in Game of 24}. 
According to Table~\ref{tab:game24_results}, OVM outperforms PRM in terms of both accuracy and correct answer proportion across all beam sizes. When compared to PRM-O, OVM demonstrates superior overall performance. Specifically, OVM achieves higher accuracy than PRM-O in 4 out of 5 scenarios at K=20, and in 4 out of 8 scenarios at K=100.

\textbf{(2) Outcome supervision holds up well in GSM8K}. 
In Table~\ref{tab:gsm8k_results_llama}, the superiority of PRM-O over PRM consistently underscores the value of focusing on outcomes. OVM's performance is closely matched with PRM, reaching higher accuracy in 6 out of 13 scenarios for K=20 and K=100. While OVM generally trails behind PRM-O across all beam sizes, the difference is typically narrow, often within a 3-point margin. 

Overall, considering the annotation costs in Appendix~\ref{app:annotate_cost}, OVM demonstrates superior utility in both settings.

\paragraph{Comparison between two supervision strategies in complete path verification} 
When evaluating the performance of complete path verification (vanilla sampling along with post-selection), it appears that process supervision does not necessarily outperform outcome supervision. This observation contrasts with previous findings, which suggested process supervision is either on par with (\citealp{deepmind-process22}) or superior to outcome supervision (\citealp{openai-process23}) in certain contexts. See the analysis below for this unexpected phenomenon:

\textbf{(1) Outcome supervision exploits shortcuts in Game of 24}. 
Table~\ref{tab:game24_results} indicates that ORM (outcome supervision) surpasses both PRM and PRM-O (process supervision). Upon closer examination of the data, we identified cases where intermediate steps were incorrect, yet the final answers were correct. These cases imply that the generator might occasionally find the right answers by chance. Process-supervised models miss these instances due to their incorrect pathways, whereas outcome-supervised models benefit from these ``shortcuts'' by prioritizing the accuracy of the final answer, irrespective of the process taken to arrive there.

\textbf{(2) Two potential factors influencing the results in GSM8K}. 
In GSM8K, ORM initially lags behind PRM and PRM-O at K=20 but outperforms them at K=100, as shown in Table~\ref{tab:gsm8k_results_llama}. This shift might be attributed to two potential factors: the presence of shortcuts and the quality of process labels. Firstly, similar to Game of 24, shortcuts also exist in GSM8K, which might explain the parallel findings by \citet{deepmind-process22} that outcome supervision and process supervision perform comparably in GSM8K. As the number of sampled paths increases, ORM's chances of exploiting a shortcut also rise, thereby enhancing its performance over PRM and PRM-O. Secondly, the discrepancy in process label quality might influence results. According to Table~\ref{tab:agreement}, the average human agreement rate is 89.3\%, while the average human-GPT4 agreement rate is 86.3\% with a difference of 3 percentage points. This underscores the complexities involved in annotating process labels.

\begin{table*}[h]
    \scriptsize
    \centering
    \begin{tabular}{cl|cc|cc|cc}
    \toprule
    \multirow{2}{*}{Sampling size} & \multirow{2}{*}{Beam size} & \multicolumn{2}{c|}{OVM} & \multicolumn{2}{c|}{PRM} & \multicolumn{2}{c}{PRM-O} \\
    & & Accuracy & Proportion & Accuracy & Proportion & Accuracy & Proportion \\
    \midrule
    \multirow{6}{*}{20} & 20$\dagger$ & 65.5\% ± 0.7\% & 32.9\% ± 0.2\% & 66.4\% ± 0.5\% & 32.9\% ± 0.2\% & 66.6\% ± 0.5\% & 32.9\% ± 0.2\% \\
     & 10 & 69.0\% ± 0.4\% & 62.3\% ± 0.6\% & 69.4\% ± 0.4\% & 63.0\% ± 0.3\% & 69.6\% ± 0.3\% & 63.2\% ± 0.6\% \\
     & 5 & 68.9\% ± 0.3\% & 65.4\% ± 0.4\% & 69.6\% ± 0.7\% & 65.9\% ± 0.6\% & 71.3\% ± 1.0\% & 67.4\% ± 0.6\% \\
     & 4 & 69.0\% ± 0.3\% & 65.7\% ± 0.9\% & 69.2\% ± 0.8\% & 65.3\% ± 0.6\% & 70.7\% ± 0.7\% & 66.4\% ± 0.3\% \\
     & 2 & 67.8\% ± 0.3\% & 65.8\% ± 0.6\% & 68.4\% ± 0.7\% & 65.9\% ± 0.7\% & 69.2\% ± 0.6\% & 66.4\% ± 0.6\% \\
     & 1 & 55.9\% ± 0.1\% & 55.2\% ± 0.1\% & 66.8\% ± 0.8\% & 65.4\% ± 0.8\% & 67.4\% ± 0.9\% & 66.0\% ± 0.9\% \\
     \midrule
    \multirow{6}{*}{50} & 50$\dagger$ & 70.1\% ± 0.2\% & 32.9\% ± 0.1\% & - & - & - & - \\
     & 25 & 72.6\% ± 0.4\% & 65.0\% ± 0.2\% & - & - & - & - \\
     & 10 & 71.2\% ± 0.3\% & 67.1\% ± 0.2\% & - & - & - & - \\
     & 5 & 71.1\% ± 0.6\% & 67.8\% ± 0.6\% & - & - & - & - \\
     & 2 & 70.1\% ± 0.8\% & 68.3\% ± 0.8\% & - & - & - & - \\
     & 1 & 55.9\% ± 0.2\% & 55.3\% ± 0.2\% & - & - & - & - \\
     \midrule
    \multirow{8}{*}{80} & 80$\dagger$ & 70.5\% ± 0.1\% & 32.8\% ± 0.1\% & - & - & - & - \\
     & 40 & 72.4\% ± 0.3\% & 65.9\% ± 0.1\% & - & - & - & - \\
     & 20 & 71.8\% ± 0.1\% & 68.6\% ± 0.3\% & - & - & - & - \\
     & 10 & 71.6\% ± 0.2\% & 68.5\% ± 0.1\% & - & - & - & - \\
     & 5 & 70.4\% ± 0.8\% & 68.2\% ± 1.0\% & - & - & - & - \\
     & 4 & 70.9\% ± 0.7\% & 68.5\% ± 0.7\% & - & - & - & - \\
     & 2 & 69.4\% ± 0.8\% & 68.0\% ± 1.1\% & - & - & - & - \\
     & 1 & 67.5\% ± 1.3\% & 66.6\% ± 1.3\% & - & - & - & - \\
     \midrule
    \multirow{9}{*}{100} & 100$\dagger$ & 71.9\% ± 0.6\% & 32.9\% ± 0.01\% & 70.8\% ± 0.7\% & 32.9\% ± 0.01\% & 71.4\% ± 0.7\% & 32.9\% ± 0.01\% \\
     & 50 & 73.8\% ± 0.4\% & 65.6\% ± 1.5\% & 72.2\% ± 0.3\% & 68.0\% ± 0.2\% & 74.2\% ± 0.4\% & 67.9\% ± 0.1\% \\
     & 25 & 73.1\% ± 0.5\% & 68.9\% ± 0.2\% & 72.1\% ± 0.2\% & 69.1\% ± 0.3\% & 74.9\% ± 0.2\% & 70.4\% ± 0.2\% \\
     & 20 & 72.1\% ± 0.5\% & 68.4\% ± 0.3\% & 72.1\% ± 0.2\% & 68.5\% ± 0.4\% & 74.3\% ± 0.4\% & 70.2\% ± 0.2\% \\
     & 10 & 71.0\% ± 0.4\% & 67.9\% ± 0.3\% & 71.3\% ± 0.4\% & 67.6\% ± 0.6\% & 73.8\% ± 0.3\% & 69.5\% ± 0.1\% \\
     & 5 & 70.1\% ± 0.3\% & 68.2\% ± 0.1\% & 70.1\% ± 0.4\% & 67.3\% ± 0.4\% & 72.8\% ± 0.5\% & 69.6\% ± 0.4\% \\
     & 4 & 70.6\% ± 0.7\% & 68.4\% ± 0.3\% & 69.4\% ± 0.6\% & 66.3\% ± 0.2\% & 72.8\% ± 0.2\% & 69.4\% ± 0.1\% \\
     & 2 & 69.0\% ± 0.5\% & 67.5\% ± 0.5\% & 68.4\% ± 0.2\% & 65.6\% ± 0.4\% & 71.9\% ± 0.1\% & 69.0\% ± 0.1\% \\
     & 1 & 67.8\% ± 0.7\% & 67.1\% ± 0.7\% & 67.0\% ± 1.0\% & 65.6\% ± 1.1\% & 71.4\% ± 0.3\% & 69.2\% ± 0.2\% \\
    \bottomrule
    \end{tabular}
    \caption{\label{tab:gsm8k_results_llama}Answer and correct answer proportion across various sampling sizes and beam sizes in GSM8K (Llama2-7B). ``Proportion'' denotes ``correct answer proportion''. `$\dagger$' denotes the setting of ``vanilla sampling + post-selection''. Due to resource constraints, experiments with PRM and PRM-O were limited to sampling sizes of K=20 and K=100.}
\end{table*}

\begin{table*}[h]
    \centering
    \begin{tabular}{cl|cc}
    \toprule
    \multirow{2}{*}{Sampling size} & \multirow{2}{*}{Beam size} & \multicolumn{2}{c}{OVM} \\
    & & Accuracy & Proportion \\
    \midrule
    \multirow{6}{*}{20} & 20$\dagger$ & 81.8\% ± 0.2\%  & 52.4\% ± 0.2\% \\
     & 10 & 82.6\% ± 0.1\% & 78.1\% ± 0.3\% \\
     & 5 & 82.1\% ± 0.4\% & 80.1\% ± 0.4\% \\
     & 4 & 82.1\% ± 0.3\% & 80.0\% ± 0.2\% \\
     & 2 & 81.7\% ± 0.3\% & 80.6\% ± 0.2\% \\
     & 1 & 80.1\% ± 0.8\% & 79.7\% ± 0.7\% \\
    \midrule
    \multirow{9}{*}{100} & 100$\dagger$ & 84.7\% ± 0.4\%  & 52.4\% ± 0.1\% \\
     & 50 & 84.7\% ± 0.4\% & 80.9\% ± 0.3\% \\
     & 25 & 84.7\% ± 0.3\% & 82.0\% ± 0.1\% \\
     & 20 & 84.3\% ± 0.1\% & 81.2\% ± 0.1\% \\
     & 10 & 84.2\% ± 0.4\% & 81.7\% ± 0.4\% \\
     & 5 & 83.0\% ± 0.4\% & 81.3\% ± 0.4\% \\
     & 4 & 83.2\% ± 0.7\% & 81.4\% ± 0.8\% \\
     & 2 & 82.0\% ± 0.1\% & 81.0\% ± 0.3\% \\
     & 1 & 81.2\% ± 0.4\% & 80.8\% ± 0.5\% \\
    \bottomrule
    \end{tabular}
    \caption{\label{tab:gsm8k_results_mistral}Answer and correct answer proportion across various sampling sizes and beam sizes in GSM8K (Mistral-7B). ``Proportion'' denotes ``correct answer proportion''. `$\dagger$' denotes the setting of ``vanilla sampling + post-selection''. Due to resource constraints, experiments were limited to sampling sizes of K=20 and K=100 with OVM.}
\end{table*}

\begin{table*}[h]
    \scriptsize
    \centering
    \begin{tabular}{cl|cc|cc|cc}
    \toprule
    \multirow{2}{*}{Sampling size} & \multirow{2}{*}{Beam size} & \multicolumn{2}{c|}{OVM} & \multicolumn{2}{c|}{PRM} & \multicolumn{2}{c}{PRM-O} \\
    & & Accuracy & Proportion & Accuracy & Proportion & Accuracy & Proportion \\
    \midrule
    \multirow{6}{*}{20} & 20$\dagger$ & 65.3\% ± 5.3\%  & 8.8\% ± 0.5\%& 60.3\% ± 4.2\%  & 8.8\% ± 0.5\% & 61.7\% ± 5.4\% & 8.8\% ± 0.5\% \\
     & 10 & 72.3\% ± 2.6\% & 16.0\% ± 0.8\% & 53.3\% ± 3.3\% & 9.2\% ± 0.4\% & 68.7\% ± 1.7\% & 15.5\% ± 0.3\% \\
     & 5 & 78.0\% ± 2.2\% & 36.8\% ± 2.7\% & 27.7\% ± 4.1\% & 8.5\% ± 1.5\% & 73.3\% ± 2.1\% & 33.7\% ± 1.9\% \\
     & 4 & 78.7\% ± 1.7\% & 46.7\% ± 1.4\% & 24.0\% ± 2.2\% & 8.4\% ± 0.4\% & 77.7\% ± 2.6\% & 42.4\% ± 0.2\% \\
     & 2 & 76.0\% ± 4.5\% & 61.7\% ± 4.3\% & 9.7\% ± 2.5\% & 6.0\% ± 1.2\% & 78.7\% ± 2.4\% & 63.3\% ± 1.9\% \\
     & 1 & 76.7\% ± 2.5\% & 76.7\% ± 2.5\% & 6.0\% ± 0.8\% & 6.0\% ± 0.8\% & 72.7\% ± 0.9\% & 73.0\% ± 0.8\% \\
    \midrule
    \multirow{6}{*}{50} & 50$\dagger$ & 86.3\% ± 3.4\%  & 8.6\% ± 0.3\%  & - & - & - & - \\
     & 25 & 90.0\% ± 0.0\% & 13.2\% ± 0.2\% & - & - & - & - \\
     & 10 & 92.7\% ± 0.9\% & 31.6\% ± 0.2\% & - & - & - & - \\
     & 5 & 89.7\% ± 0.5\% & 56.8\% ± 0.2\% & - & - & - & - \\
     & 2 & 87.3\% ± 0.5\% & 76.3\% ± 0.2\% & - & - & - & - \\
     & 1 & 84.7\% ± 1.2\% & 84.7\% ± 1.2\% & - & - & - & - \\
    \midrule
    \multirow{8}{*}{80} & 80$\dagger$ & 95.7\% ± 1.9\%  & 8.5\% ±  0.3\% & - & - & - & -  \\
     & 40 & 95.0\% ± 0.0\% & 12.0\% ± 0.0\% & - & - & - & - \\
     & 20 & 96.0\% ± 0.0\% & 20.6\% ± 0.4\% & - & - & - & - \\
     & 10 & 97.3\% ± 0.5\% & 40.2\% ± 0.9\% & - & - & - & - \\
     & 5 & 92.0\% ± 0.8\% & 63.3\% ± 1.2\% & - & - & - & - \\
     & 4 & 92.3\% ± 0.5\% & 70.0\% ± 1.1\% & - & - & - & - \\
     & 2 & 88.7\% ± 0.9\% & 79.0\% ± 0.4\% & - & - & - & - \\
     & 1 & 85.0\% ± 2.2\% & 85.0\% ± 2.2\% & - & - & - & - \\
    \midrule
    \multirow{9}{*}{100} & 100$\dagger$ & 95.3\% ± 0.5\%  & 8.6\% ± 0.1\% & 93.3\% ± 0.9\%  & 8.6\% ± 0.1\% & 93.3\% ± 0.9\% & 8.6\% ± 0.1\% \\
     & 50 & 94.3\% ± 1.7\% & 13.2\% ± 0.6\% & 88.7\% ± 0.5\% & 7.7\% ± 0.2\% & 94.3\% ± 0.9\% & 13.3\% ± 0.2\% \\
     & 25 & 98.3\% ± 1.2\% & 18.7\% ± 0.5\% & 76.7\% ± 2.9\% & 7.8\% ± 0.4\% & 94.7\% ± 1.9\% & 17.2\% ± 0.9\% \\
     & 20 & 95.7\% ± 0.9\% & 22.7\% ± 0.4\% & 65.7\% ± 2.9\% & 7.3\% ± 0.3\% & 95.0\% ± 1.4\% & 21.3\% ± 0.7\% \\
     & 10 & 97.7\% ± 0.5\% & 43.3\% ± 0.5\% & 35.3\% ± 3.4\% & 6.8\% ± 0.7\% & 95.3\% ± 1.2\% & 40.0\% ± 0.9\% \\
     & 5 & 93.3\% ± 1.2\% & 66.3\% ± 0.9\% & 23.0\% ± 1.6\% & 6.3\% ± 0.6\% & 96.7\% ± 0.5\% & 64.6\% ± 0.9\% \\
     & 4 & 91.3\% ± 0.5\% & 70.5\% ± 0.9\% & 17.3\% ± 2.5\% & 6.0\% ± 0.7\% & 95.3\% ± 0.9\% & 68.9\% ± 0.8\% \\
     & 2 & 89.3\% ± 1.9\% & 80.7\% ± 1.6\% & 7.0\% ± 1.4\% & 5.3\% ± 0.8\% & 91.0\% ± 0.0\% & 79.7\% ± 1.7\% \\
     & 1 & 84.3\% ± 0.5\% & 84.3\% ± 0.5\% & 4.7\% ± 1.2\% & 4.7\% ± 1.2\% & 84.3\% ± 0.9\% & 84.7\% ± 1.2\% \\
    \bottomrule
    \end{tabular}
    \caption{\label{tab:game24_results}Answer and correct answer proportion across various sampling sizes and beam sizes in Game of 24. ``Proportion'' denotes ``correct answer proportion''. `$\dagger$' denotes the setting of ``vanilla sampling + post-selection''. Due to resource constraints, experiments with PRM and PRM-O were limited to sampling sizes of K=20 and K=100.}
\end{table*}



\begin{figure*}[ht]
    \centering
    \subfigbottomskip=2pt
    \subfigcapskip=-2pt
    \subfigure[Accuracy in GSM8K]{
        \includegraphics[width=0.45\linewidth]{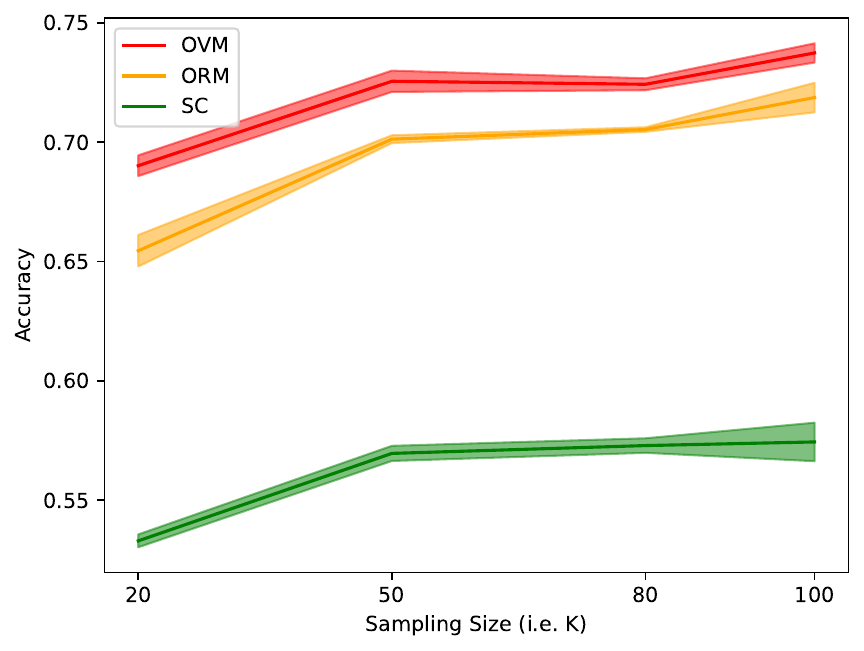}}
    \subfigure[Accuracy in Game of 24]{
        \includegraphics[width=0.45\linewidth]{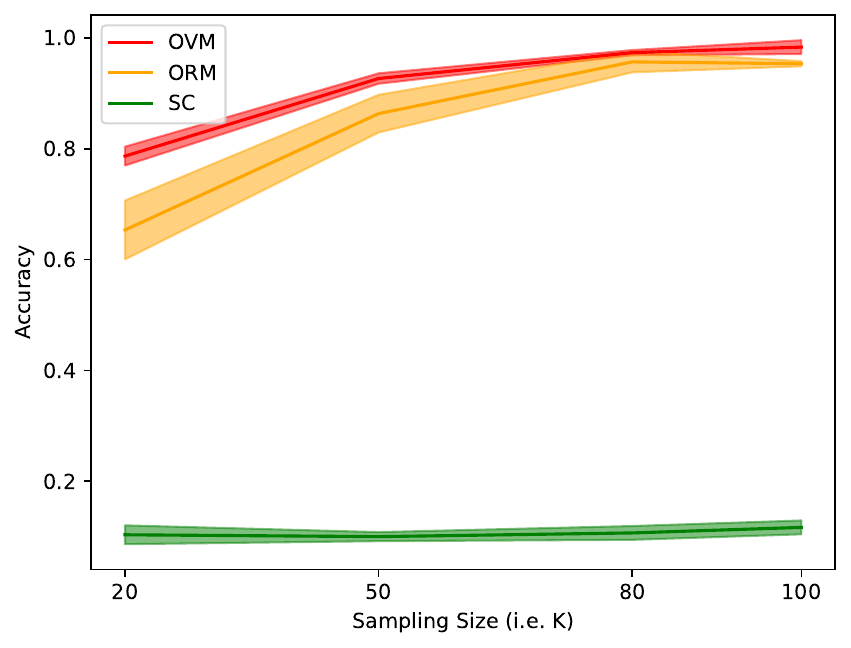}}
    \\
    \subfigure[Correct answer proportion in GSM8K]{
        \includegraphics[width=0.45\linewidth]{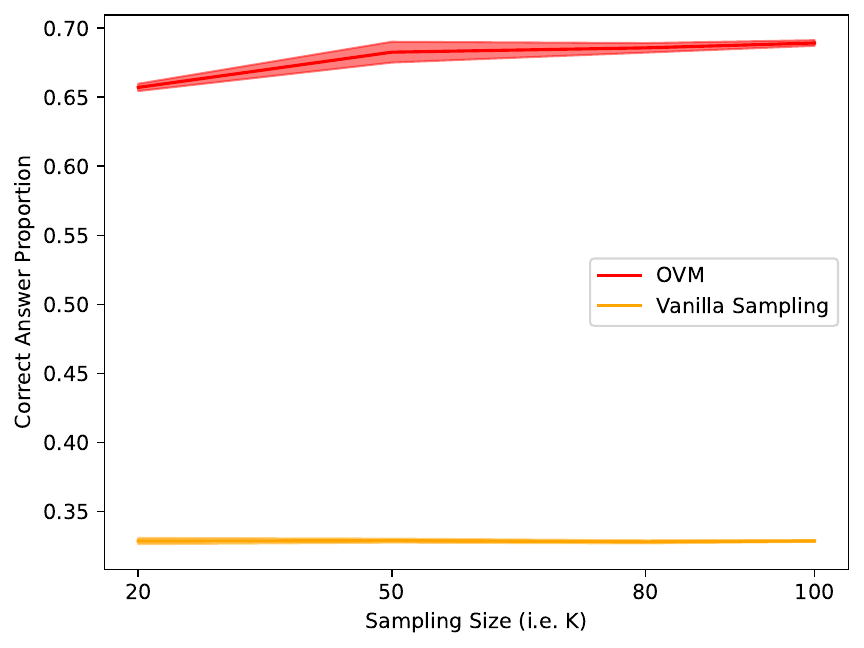}}
    \subfigure[Correct answer proportion in Game of 24]{
        \includegraphics[width=0.45\linewidth]{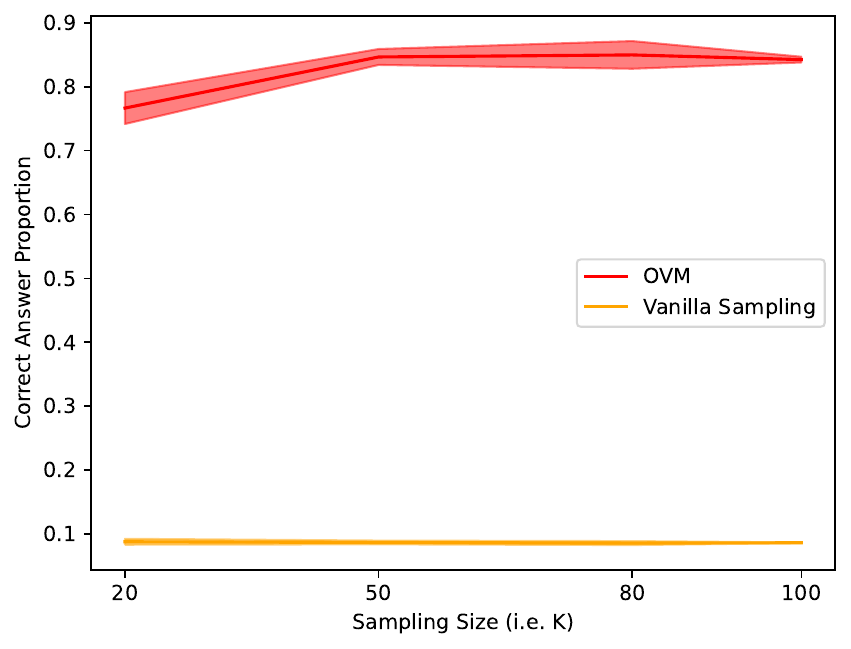}}
    \caption{\label{fig:all_k_acc_proportion} The tendency of accuracy and correct answer proportion with respect to the sampling size (Llama2-7B)}
\end{figure*}

\end{document}